\documentclass[12pt]{article} % For LaTeX2e

\usepackage[margin=1in]{geometry}

\usepackage{microtype}
\usepackage{graphicx}
\usepackage{subfigure}
\usepackage{booktabs} % for professional tables
\usepackage{amsfonts,amscd,amssymb}
\usepackage{url}
\usepackage{amsthm,amsmath,natbib}
\usepackage{hyperref}

\newtheorem{Theorem}{Theorem}[section]
\newtheorem{Definition}[Theorem]{Definition}

\newtheorem{Assumption}[Theorem]{Assumption}

\newtheorem{Lemma}[Theorem]{Lemma}

\newcommand{\mc}{\mathcal}
\newcommand{\mbb}{\mathbb}
\newcommand{\argmin}{\operatornamewithlimits{argmin}}

\newcommand{\vectorize}{\textbf{vec}}

\title{Searching for Minimal Optimal Neural Networks}
\date{\today}
\author{
Lam Si Tung Ho \\
Department of Mathematics and Statistics \\
Dalhousie University, Halifax, Nova Scotia, Canada
\and
Vu Dinh \\
Department of Mathematical Sciences, University of Delaware\\
Delaware, USA
}

\begin{document}
\maketitle

\begin{abstract}
Large neural network models have high predictive power but may suffer from overfitting if the training set is not large enough.
Therefore, it is desirable to select an appropriate size for neural networks. 
The destructive approach, which starts with a large architecture and then reduces the size using a Lasso-type penalty, has been used extensively for this task.
Despite its popularity, there is no theoretical guarantee for this technique. 
Based on the notion of minimal neural networks, we posit a rigorous mathematical framework for studying the asymptotic theory of the destructive technique.
We prove that Adaptive group Lasso is consistent and can reconstruct the correct number of hidden nodes of one-hidden-layer feedforward networks with high probability.
To the best of our knowledge, this is the first theoretical result establishing for the destructive technique.
\end{abstract}

 \clearpage
\section{Introduction}

Artificial neural networks are highly expressive models that achieve excellent performance on many tasks.
However, the performance of a neural network model depends heavily on its structure. 
In particular, training with oversized neural networks on small or moderate datasets can lead to overfitting.
Moreover, training large neural networks requires a high memory and computation cost.
Therefore, choosing the right size for neural networks is an important problem and has been studied intensively.

Two common approaches for this task are the constructive method \cite{bello1992enhanced} and the destructive technique \cite{lecun1989optimal}.
The constructive method starts with a small neural network and gradually incorporates additional components until finding the best architecture.
The destructive technique, on the other hand, starts with a large neural network and remove unimportant components. 
One drawback of the constructive method is that we have to train a new model each time we add a new component.
On the contrary, we can utilize a Lasso-type penalty for the destructive technique to avoid this pitfall.
Although this method has been used extensively, to the best of our knowledge, no theoretical result has been established, even for one-hidden-layer feedforward networks.
The main challenge is that neural network models are non-linear and unidentifiable.

Choosing the architecture of neural networks can be considered as a model selection procedure.
To study asymptotic properties of such procedures, we posit a rigorous mathematical framework using the notion of \emph{minimal neural networks}.
For simplicity, we focus on one-hidden-layer feedforward networks with hyperbolic tangent activation function, which will be called ``networks" from now on.
In particular, a function $f: \mbb{R}^d \to \mbb{R}$ is a network if
\[
f(x) = v^\top \cdot \tanh(u \cdot x + b_1) + b_2
\]
where  $u$ is a $H \times d$ matrix; $b_1, v$ are $H$-dimensional vectors; and $b_2$ is a real number.
Here, $v^\top$ is the transpose of $v$; $H$ is the number of nodes in the hidden layer, $u, v$ are the weights; and $b_1, b_2$ are the biases.
A network is called \emph{minimal} if it does not have the same input-output map with a network with fewer hidden nodes.
The best structure for the training network is the structure of an optimal network that is also minimal.
It is obvious that there exists such a network if there exists an optimal network.
Recall that a network is \emph{optimal} if it minimizes the expected risk.
Therefore, we need to search for a minimal network among optimal networks.
We call it \emph{a minimal optimal network}.

A popular penalty for automatically reducing the number of hidden nodes of neural networks is the Group Lasso \cite{murray2015auto, alvarez2016learning, scardapane2017group, huang2018condensenet, murray2019autosizing}.
The penalty groups the weights and the bias parameters of each hidden node together and shrinks them to zero simultaneously.
Recent empirical studies in \cite{ho2020consistent,dinh2020consistent} suggest that the group Lasso penalty may not be as efficient as the Adaptive group Lasso for selecting between neural network models.
In this paper, we propose an Adaptive group Lasso method for the destructive technique and prove that the proposed method is guaranteed to recover the architecture of minimal optimal networks with high probability.
We use a simple simulation to illustrate that Adaptive group Lasso may be more advantageous than group Lasso in selecting the number of hidden nodes of networks.

\paragraph{Related work:}
Rynkiewicz \cite{rynkiewicz2006consistent} proposed an information criterion that can consistently select the number of hidden nodes of a network.
Nevertheless, information criterion has little application in practice because it requires to be computed for all possible models.
The performance of the destructive technique has been investigated extensively using both synthetic and real data \cite{lecun1989optimal,murray2015auto, alvarez2016learning, scardapane2017group, huang2018condensenet}.
However, little work has been done to investigate theoretically the performance of these methods for neural network models although asymptotic properties of Lasso-type regularization methods have been studied extensively for linear model \cite{zou2006adaptive, zhao2006model, wang2008note, meinshausen2009lasso, liu2009estimation}.
Notably, there are some recent theoretical works on asymptotic properties of Lasso-type regularization methods for feature selection under neural networks models \cite{dinh2020consistent,feng2017sparse, farrell2018deep, fallahgoul2019towards, shen2019asymptotic}.

\section{Mathematical framework}

Let $\alpha = (u, v, b_1, b_2)$, we denote a network with weights $u, v$ and biases $b_1, b_2$ by $f_\alpha$, and the $u, v, b_1, b_2$ components of $\alpha$ by $u_\alpha, v_\alpha, b_{1_\alpha}, b_{2_\alpha}$ respectively.
Recall that a network is \emph{minimal} if it does not have the same input-output map with a network with fewer hidden nodes.
Sussmann \cite{sussmann1992uniqueness} provided the following necessary and sufficient conditions for a network $f_\alpha$ to be minimal:

\begin{Lemma}
A network $f_\alpha$ is minimal if and only if
\begin{enumerate}
\item[(i)] $u_\alpha^{[:,i]} \ne 0$ for all $i$
\item[(ii)] $v_\alpha^{[i]} \ne 0$ for all $i$
\item[(iii)] $(u_\alpha^{[:,i]}, b_{1_\alpha}^{[i]}) \ne \pm (u_\alpha^{[:,j]}, b_{1_\alpha}^{[j]})$ for all $i \ne j$,
\end{enumerate}
where $u_\alpha^{[:,i]}$ is the $i$-th column of the matrix $u_\alpha$ and $v_\alpha^{[i]}$ is the $i$-th component of the vector $v_\alpha$.
\label{lem:minimal}
\end{Lemma}

Moreover, Sussmann \cite{sussmann1992uniqueness} showed that two minimal networks that have same input-output map can be transformed from one to another by a series of sign-flip and node-interchange transformations.
A sign-flip transformation changes $(u_\alpha^{[:,i]}, b_{1_\alpha}^{[i]}) \to (- u_\alpha^{[:,i]}, - b_{1_\alpha}^{[i]})$ for a node $i$ while a node-interchange transformation switches the labels of two nodes: $(i, j) \to (j, i )$.
Therefore, a set of minimal networks that have the same input-output map is always finite.

We assume that the training data $(Y_k, X_k)_{k=1}^n$ are generated from the following model
\begin{equation}
Y_k = f_{\alpha^*}(X_k) + \epsilon_k
\label{eqn:model}
\end{equation}
where $\{\epsilon_i\}_{k = 1}^n$ are i.i.d. random variables that follow the normal distribution $\mc{N}(0, \sigma^2)$.
Without loss of generality, we assume that the data-generating network $f_{\alpha^*}$ is a minimal network.
Let $H^*$ be the unknown number of hidden nodes of $f_{\alpha^*}$. 
We consider the destructive techniques proposed by \cite{alvarez2016learning} for reconstructing $f_{\alpha^*}$.
Specifically, we fit a large network with $H$ hidden nodes and remove redundant nodes using Adaptive group Lasso.
Throughout this paper, we make the following assumption: 

\begin{Assumption}
We assume that $H^* \leq H$ and $\| \vectorize(\alpha^*) \|_\infty$ is bounded by a constant $W$.
Here, $\|\cdot\|_\infty$ is the $\ell_\infty$ norm and $\vectorize(\cdot)$ is the vectorization operator.
\end{Assumption}

Let $\mc{H}$ be the parameter space of all networks that have $H$ hidden nodes such that $\| \vectorize(\alpha) \|_\infty \leq W$ for all $\alpha \in \mc{H}$, and $\mc{H}^*$ be the parameter space of all minimal networks that have the same input-output map with $f_{\alpha^*}$.
From now on, ``a network in $\mc{H}$" means its parameter is in $\mc{H}$.
A hidden node $i$ is called a \emph{zero node} of $f_\alpha$ if $(u_{\alpha}^{[:,i]}, v_{\alpha}^{[i]}, b_{1_{\alpha}}^{[i]}) = 0$, and it is called a \emph{non-significant node} of $f_\alpha$ if $u^{[:,i]} = 0$ or $v^{[i]} = 0$.
Let $\mc{\overline H}^*$ be the parameter space of networks in $\mc{H}$ such that if we remove all zero nodes of a network in $\mc{\overline H}^*$, we obtain a network in $\mc{H}^*$.
Similarly, let $\mc{K}$ be the parameter space of networks in $\mc{H}$ such that if we remove all non-significant nodes of a network $\mc{K}$, we obtain a network in $\mc{H}^*$.
Finally, let $\mc{Q}$ be the parameter space of networks in $\mc{H}$ that have the same input-output map with $f_{\alpha^*}$.
It is obvious that $\mc{\overline H}^* \subset \mc{K} \subset \mc{Q} \subset \mc{H}$.
We can think of $\mc{\overline H}^*$ as an embedding of $\mc{H}^*$ into $\mc{H}$.
Since $\mc{H}^*$ is finite, $\mc{\overline H}^*$ is also finite.

For $\alpha,\beta \in \mc{H}$, we define a distance between $\alpha$ and $\beta$ by
\[
d(\alpha, \beta) = \| \vectorize(\alpha) - \vectorize(\beta) \|
\]
where $\|\cdot\|$ is the $\ell_2$ norm.
With this notation, we can rigorously define \emph{consistency} and \emph{model selection consistency}.

\begin{Definition}
An estimator $\hat \alpha$ is consistent if $d(\hat \alpha, \mc{\overline H}^*) \to 0$ in probability.
\end{Definition}

\begin{Definition}
An estimator $\hat \alpha$ is model selection consistent if 
\begin{itemize}
\item $\hat \alpha$ is consistent
\item for any $\delta > 0$, there exists $N_\delta$ such that if $n \ge N_\delta$, the probability that $f_{\hat \alpha}$ is a minimal network with $H^*$ non-zero nodes is at least $1 - \delta$.
\end{itemize}
\end{Definition}

Next, we introduce the Adaptive group Lasso method for estimating the unknown parameter $\alpha$.
It is a two-step process: 
\begin{itemize}
\item Step 1: obtain an initial estimator using Group Lasso
\[
\hat{\alpha}^{GL}_n = \argmin_{\alpha \in \mc{H}}{\frac{1}{n} \sum_{k = 1}^n{[Y_k - f_{\alpha}(X_k)]^2}} + \zeta_n \sum_{i=1}^H{\|w_i\|}
\]
where $w_i = (u^{[:,i]}, v^{[i]}, b_1^{[i]})$ is the vector of all parameters that associated with the $i$-th hidden node; $\zeta_n > 0$ is the regularizing parameter.
%Note that \cite{murray2015auto} proposed to use $(u^{[:,i]}, b_1^{[i]})$ as a group rather than $w_i$.
%We prefer to shrink $u^{[:,i]}, v^{[i]}, b_1^{[i]}$ together because when $(u^{[:,i]}, b_1^{[i]}) = 0$, $v^{[i]}$ does not affect the output of the network anymore.
%In practice, the two choices should give similar results.
%This observation is confirmed by our simulations in Section \ref{sec:sim}.

\item Step 2: compute the Adaptive group Lasso estimator
\[
\hat{\alpha}_n = \argmin_{\alpha \in \mc{H}}{\frac{1}{n} \sum_{k = 1}^n{[Y_k - f_{\alpha}(X_k)]^2}} + \lambda_n \sum_{i=1}^H {\frac{\|w_i\|}{\|w_{i_{\hat{\alpha}^{GL}_n}}\|^\gamma}}
\]
where $\gamma$ is a positive number; $\lambda_n > 0$ is the regularizing parameter; and $w_{i_{\hat{\alpha}^{GL}_n}}$ is the $w_i$-component of $\hat{\alpha}^{GL}_n$.
Here, we use the convention that $0/0 = 0$.
\end{itemize}

For convenience, we define
\[
L(\alpha) = \sum_{i=1}^H{\|w_i\|}, \quad M_n(\alpha) = \sum_{i=1}^H {\frac{\|w_i\|}{\|w_{i_{\hat{\alpha}^{GL}_n}}\|^\gamma}}.
\]

Next, we provide some basic lemmas that we need for our proofs.
To ease the notations, we will use $C_k$, $k = 1, 2, \ldots$, for denoting generic constants.

\begin{Lemma}
We have $| L(\alpha) - L(\beta) | \leq \sqrt{H} d(\alpha, \beta), \quad \forall \alpha \in \mc{H}.$
\label{lem:LipL}
\end{Lemma}

\begin{proof}
$| L(\alpha) - L(\beta) | = \sum_{i=1}^H{ | \|w_{i_\alpha}\| - \|w_{i_\beta}\| |} \leq \sum_{i=1}^H{ \|w_{i_\alpha} -w_{i_\beta}\|} \leq \sqrt{H} d(\alpha, \beta).$
\end{proof}

\begin{Lemma}
There exist $C_2 > 0$ such that for any $\beta \in \mc{Q}$, there exists $\beta^* \in \mc{\overline H}^*$ such that
$L(\beta) - L(\beta^*) \geq \min \left \{C_2, d(\beta, \mc{\overline H}^*) \right \}.$
\label{lem:dominateL}
\end{Lemma}

\begin{proof}
Consider $\alpha^* \in \mc{\overline H}^*$.
Let $\mc{I}$ be the set of non-zero nodes of $f_{\alpha^*}$.
For $i \in \mc{I}$, we define $A_i = \{ j : (u^{[:,j]}_{\beta}, b_{1_{\beta}}^{[j]}) = \pm (u^{[:,i]}_{\alpha^*}, b_{1_{\alpha^*}}^{[i]}) \}$.
Since $\beta \in  \mc{Q}$, $A_i \ne \emptyset$ for all $i \in \mc{I}$.
Set $\mc{J} = \{ \min A_i : i \in \mc{I}\}$. 
We define $\beta^* \in \mc{H}$ as follows:
\begin{itemize}
\item For $j \in \mc{J}$, 
\begin{equation}
\begin{aligned}
&(u_{\beta^*}^{[:,j]}, b_{1_{\beta^*}}^{[j]}) =  (u_{\alpha^*}^{[:,i]}, b_{1_{\alpha^*}}^{[i]})\\
&v_{\beta^*}^{[j]} = \sum_{k \in A_i}{\left (2 I_{\left \{(u_{\beta}^{[:,k]}, b_{1_{\beta}}^{[k]}) =  (u_{\alpha^*}^{[:,i]}, b_{1_{\alpha^*}}^{[i]}) \right \} } - 1 \right ) v_{\alpha^*}^{[k]}}
\end{aligned}
\label{eq:step1}
\end{equation}
where $i \in \mc{I}$ such that $j = \min A_i$, and $I_{\{ \cdot \}}$ is the indicator function.
\item For all $j \not\in \mc{J}$, 
\begin{equation}
w_{j_{\beta^*}} = 0.
\label{eq:step2}
\end{equation}
\end{itemize}
Here, the idea is to merge all identical nodes in $\beta$ corresponding to a non-zero node of $f_{\alpha^*}$ into one.
By Lemma \ref{lem:minimal}, we have $\beta^* \in \mc{\overline H}^*$.
From the construction of $\beta^*$, $\mc{J}$ is the set of all non-zero nodes of $f_{\beta^*}$ and $u_{\beta^*}^{[:,j]} = u_{\beta}^{[:,j]}$ for all $j \in \mc{J}$.

If $\min A_i \ne \max A_i$ for a node $i$, then we need to merge at least some identical nodes of $\beta$ to construct $\beta^*$.
It is worth noticing that for $|x|, |y_1|, |y_2|, \ldots, |y_K| \leq C$, we have
%\[
%\sqrt{x^2 + y^2} + \sqrt{x^2 + z^2} - \sqrt{x^2 + (y+z)^2} \geq \frac{x^2}{(2\sqrt{2} + \sqrt{5})C}.
%\]
\[
\sum_{k=1}^K{\sqrt{x^2 + y_k^2}} - \sqrt{x^2 + \left (\sum_{k=1}^K{y_k} \right )^2} \geq \frac{(K-1)x^2}{(2\sqrt{2} + \sqrt{5})C}.
\]
Applying this inequality for $x^2 = \|u_{\alpha^*}^{[:,i]}\|^2 + \|b_{1_{\alpha^*}}^{[i]}\|^2$ and $y_k = v_{\alpha^*}^{[k]}$,
we conclude that there exists a constant $C_2 > 0$ such that
\begin{equation}
L(\beta) - L(\beta^*) \geq C_2.
\label{eq:lowerL}
\end{equation}
This is because $\mc{I}$ is a finite set and $\|u_{\alpha^*}^{[:,i]}\|^2 > 0$.
In other words, merging nodes increases the group Lasso penalty significantly.

On the other hand, if $\max A_i = \min A_i$ for all $i \in \mc{I}$.
By the construction of $\beta^*$, 
\begin{equation}
L(\beta) - L(\beta^*) = d(\beta, \beta^*) \geq d(\beta, \mc{\overline H}^*).
\label{eq:eqL}
\end{equation}
Combining \eqref{eq:lowerL} and \eqref{eq:eqL}, we have $L(\beta) - L(\beta^*) \geq \min \left \{C_2, d(\beta, \mc{\overline H}^*) \right \}.$
\end{proof}

\begin{Lemma}
There exists a constant $C_1 > 0$ such that for any $\beta \in \mc{Q}$, if 
the set $\mc{U} = \{ i: \|w_{i_{\beta}}\| < C_1\}$ has $H - H^*$  nodes, then the network $\beta'$, obtained by setting the weights and the biases of $\beta$ at nodes in $\mc{U}$ to zero, belongs to $\mc{\overline H}^*$.
\label{lem:collapsable}
\end{Lemma}

\begin{proof}
Recall that $\alpha^*$ is a network in $\mc{H}^*$.
Denote
\begin{align*}
A &= \{ \| u_{i_{\alpha^*}} \| :  i = 1, 2, \ldots, H^* \} \\
A_i &= \{ j : (u^{[:,j]}_{\beta}, b_{1_{\beta}}^{[j]}) = \pm (u^{[:,i]}_{\alpha^*}, b_{1_{\alpha^*}}^{[i]}) \}, \quad i = 1, 2, \ldots, H^*.
\end{align*}
Note that $|A| = H^*$.
Set $C_1 = \min A > 0$.
By the definition of $\mc{U}$, we have
\[
A_i \cap \mc{U} = \emptyset, \quad \forall i = 1, 2, \ldots, H^*.
\]
Since $| \mc{U} | = H - H^*$ and $A_i \ne \emptyset$, we have $| A_i | = 1$ for all $i = 1, 2, \ldots, H^*$.

Let $\beta^* \in \mc{\overline H}^*$ be the network constructed from $\beta$ as in \eqref{eq:step1} and \eqref{eq:step2}.
From the fact that $| A_i | = 1$ for all $i = 1, 2, \ldots, H^*$, we obtain
\[
w_{i_{\beta^*}} = 
\begin{cases}
w_{i_{\beta}} & i \not\in \mc{U} \\
0 & i \in \mc{U}.
\end{cases}
\]
Thus, $\beta' = \beta^* \in \mc{\overline H}^*$.
\end{proof}

\section{Asymptotic properties of Adaptive group Lasso}

Let $R_n(\alpha)$ and $R(\alpha)$ be the empirical risk and expected risk of $f_\alpha$ respectively.
That is, $R_n(\alpha) = \frac{1}{n} \sum_{i=1}^n{[Y_i - f_{\alpha}(X_i)]^2}$ and $R(\alpha) = E([Y - f_{\alpha}(X)]^2).$
Note that $R(\alpha) = \sigma^2 = \min_{\beta \in \mc{H}} R(\beta)$ if and only if $\alpha \in \mc{Q}$.

\subsection{Properties of risks}

First, we state some properties of $R_n(\alpha)$ and $R(\alpha)$.

\begin{Lemma}[Generalization bound]
For any $\delta>0$, there exist $C_3(\delta) > 0$ 
\[
| R_n( \alpha) - R(\alpha) | \le C_3 \frac{\log n}{\sqrt{n}}, \quad \forall \alpha \in \mc{H}
\]
with probability at least $1-\delta$.
\label{lem:generalization}
\end{Lemma}

\begin{proof}
This Lemma is a direct consequence of Lemma 3.3 in \cite{dinh2020consistent}.
\end{proof}

\begin{Lemma}
There exists $C_4, \nu > 0$ such that
\[
R(\alpha) - \sigma^2 \geq C_4 d(\alpha, \mc{Q})^\nu, \quad \forall \alpha \in \mc{H}.
\]
\label{lem:ineq}
\end{Lemma}

\begin{proof}
Since $\mc{H}$ is a compact set and $\mc{Q}$ is the zero set of $R(\alpha) - \sigma^2$, we can obtain this lower bound by applying {\L}ojasiewicz inequality \cite{ji1992global}.
\end{proof}

\begin{Lemma}[Lipschitzness]
There exists $C_5 > 0$ such that
\[
| R(\alpha) - R(\beta) | \leq C_5 d(\alpha, \beta) ,\quad \forall \alpha, \beta \in \mc{H}.
\]
For any $\delta > 0$, there exists $C_6(\delta)$ such that
\[
| R_n(\alpha) - R_n(\beta) | \leq C_6  d(\alpha, \beta), \quad \forall \alpha, \beta \in \mc{H}
\] 
with probability at least $1 - \delta$.
\label{lem:Lip}
\end{Lemma}

\begin{proof}
The proof can be found in Section 5.3 in \cite{dinh2020consistent}.
\end{proof}

\subsection{Model selection consistency of Adaptive group Lasso}

Now, we are ready to prove that Adaptive group Lasso is structural consistent.
The first step is deriving the convergence rate of group Lasso.

\begin{Theorem}
Assume that $\zeta_n \to 0$ and $\zeta_n \sqrt{n}/ \log(n) \to \infty$.
For any $\delta >0$, when $n$ is sufficiently large, we have 
\[
d(\hat{\alpha}^{GL}_n, \mc{\overline H}^*) \leq 2  C_3 \frac{\log n}{\zeta_n \sqrt{n}} + (1+\sqrt{H}) \left ( \frac{4C_3}{C_4} \frac{\log n}{\sqrt{n}} + \frac{2C_7}{C_4} \zeta_n^{\nu/(\nu - 1)} \right )^{1/\nu}
\]
with probability at least $1 - \delta$.
\label{thm:gl}
\end{Theorem}

\begin{proof}
Since $\mc{Q}$ is a closed set, there exists $\beta_n \in \mc{Q}$ such that $d(\hat{\alpha}^{GL}_n, \beta_n) = d(\hat{\alpha}^{GL}_n, \mc{Q})$.
Applying Lemma \ref{lem:LipL}, \ref{lem:generalization}, and \ref{lem:ineq}, for $\delta > 0$, we have
\begin{align*}
C_4 d(\hat{\alpha}^{GL}_n, \beta_n)^\nu &\leq R(\hat{\alpha}^{GL}_n) - R(\beta_n) \leq 2  C_3 \frac{\log n}{\sqrt{n}} + R_n(\hat{\alpha}^{GL}_n) - R_n(\beta_n) \\
& \leq 2  C_3 \frac{\log n}{\sqrt{n}} + \zeta_n \left[ L(\beta_n) - L(\hat{\alpha}^{GL}_n) \right] \\
&\leq 2  C_3 \frac{\log n}{\sqrt{n}}  + \zeta_n \sqrt{H} d(\hat{\alpha}^{GL}_n, \beta_n)
\end{align*}
with probability at least $1 - \delta$.
The last inequality holds because of the fact that $\ell_1$-penalty is a Lipschitz function.
By Young's inequality,
\[
\zeta_n \sqrt{H} d(\hat{\alpha}^{GL}_n, \beta_n) \leq C_7 \zeta_n^{\nu/(\nu - 1)} + \frac{C_4}{2} d(\hat{\alpha}^{GL}_n, \beta_n)^\nu.
\]
Hence
\begin{equation}
d(\hat{\alpha}^{GL}_n, \beta_n)^\nu \leq \frac{4C_3}{C_4} \frac{\log n}{\sqrt{n}} + \frac{2C_7}{C_4} \zeta_n^{\nu/(\nu - 1)}
\label{eq:1}
\end{equation}
with probability at least $1 - \delta$.

By Lemma \ref{lem:dominateL}, there exists $\beta^*_n \in \mc{\overline H}^*$ such that
\[
L(\beta_n) - L(\beta^*_n) \geq \min \left \{C_2, d(\beta_n, \mc{\overline H}^*) \right \}.
\]
We have
\begin{align*}
L(\beta^*_n) - L(\hat \alpha_{n_k}^{GL}) &= L(\beta^*_n) - L(\beta_n) + L(\beta_n) - L(\hat \alpha_{n_k}^{GL}) \\
&\leq - \min \left \{C_2, d(\beta_n, \mc{\overline H}^*) \right \} +  \sqrt{H} d(\hat{\alpha}^{GL}_n, \beta_n).
\end{align*}
By Lemma \ref{lem:generalization}, for $\delta >0 $,
\begin{align*}
0 \leq R(\hat \alpha_n^{GL})  - R( \beta^*_n) &\le R_n(\hat \alpha_n^{GL}) - R_n(\beta^*_n) + 2  C_3 \frac{\log n}{\sqrt{n}} \\
&\le 2  C_3 \frac{\log n}{\sqrt{n}} + \zeta_n \left[ L(\beta^*_n) - L(\hat \alpha_n^{GL})\right]
\end{align*}
with probability at least $1 - \delta$.
Therefore
\begin{align*}
0 & \leq 2  C_3 \frac{\log n}{\zeta_{n} \sqrt{n}} + L(\beta^*_n) - L(\hat \alpha_{n}^{GL}) \leq 2  C_3 \frac{\log n}{\zeta_{n} \sqrt{n}} - \min \left \{C_2, d(\beta_n, \mc{\overline H}^*) \right \} +  \sqrt{H} d(\hat{\alpha}^{GL}_n, \beta_n).
\end{align*}
By equation \eqref{eq:1}, 
\[
\min \left \{C_2, d(\beta_n, \mc{\overline H}^*) \right \} \leq 2  C_3 \frac{\log n}{\zeta_{n} \sqrt{n}} + \sqrt{H}  \left ( \frac{4C_3}{C_4} \frac{\log n}{\sqrt{n}} + \frac{2C_7}{C_4} \zeta_n^{\nu/(\nu - 1)} \right )^{1/\nu}
\]
with probability at least $1 - \delta$.
Thus, for sufficiently large $n$, we get
\begin{equation}
d(\beta_n, \mc{\overline H}^*) \leq 2  C_3 \frac{\log n}{\zeta_{n} \sqrt{n}} + \sqrt{H}  \left ( \frac{4C_3}{C_4} \frac{\log n}{\sqrt{n}} + \frac{2C_7}{C_4} \zeta_n^{\nu/(\nu - 1)} \right )^{1/\nu}
\label{eq:addbound}
\end{equation}
with probability at least $1 - \delta$.

From \eqref{eq:1} and \eqref{eq:addbound}, we conclude that, for sufficiently large $n$,
\begin{align*}
d(\hat{\alpha}^{GL}_n, \mc{\overline H}^*) &\leq d(\hat{\alpha}^{GL}_n, \beta_n) + d(\beta_n, \mc{\overline H}^*) \\
&\leq 2  C_3 \frac{\log n}{\zeta_{n} \sqrt{n}} + (1 + \sqrt{H})  \left ( \frac{4C_3}{C_4} \frac{\log n}{\sqrt{n}} + \frac{2C_7}{C_4} \zeta_n^{\nu/(\nu - 1)} \right )^{1/\nu}
\end{align*}
with probability at least $1 - \delta$.
\end{proof}

A direct consequence of Theorem \ref{thm:gl} is that the group Lasso estimator is consistent.
In the second step, we prove that the Adaptive group Lasso estimate has at most $H^*$ non-zero nodes.

\begin{Theorem}
Assume that $\zeta_n \to 0$, $\zeta_n \sqrt{n}/ \log(n) \to \infty$, $\lambda_n \to 0$, and
\[
\frac{\lambda_n^{1/\gamma}}{2  C_3 \frac{\log n}{\zeta_n \sqrt{n}} + \sqrt{H} \left ( \frac{4C_3}{C_4} \frac{\log n}{\sqrt{n}} + \frac{2C_7}{C_4} \zeta_n^{\nu/(\nu - 1)} \right )^{1/\nu}} \to \infty.
\]
For any $\delta >0$, when $n$ is sufficiently large, $f_{\hat \alpha_n}$ has at most $H^*$ non-zero nodes with probability at least $1 - \delta$.
\label{thm:nodes}
\end{Theorem}

\begin{proof}
Let $\alpha^*_n \in \mc{\overline H}^*$ such that $d(\hat{\alpha}^{GL}_n, \alpha^*_n) = d(\hat{\alpha}^{GL}_n, \mc{\overline H}^*)$.
We denote the set of zero nodes of $f_{\alpha^*_n}$ by $\mc{I}_n$. 
Let $\pi_n : \mc{H} \to \mc{H}$ be a function that set all the weights and biases associated with nodes in $\mc{I}_n$ to $0$.
That is,
\[
w_{i_{\pi_n(\alpha)}} = (u_{\pi_n(\alpha)}^{[:,i]}, v_{\pi_n(\alpha)}^{[i]}, b_{1_{\pi_n(\alpha)}}^{[i]}) = 
\begin{cases}
0 & \text{if } i \in \mc{I}_n \\
(u_{\alpha}^{[:,i]}, v_{\alpha}^{[i]}, b_{1_{\alpha}}^{[i]}) & \text{if } i \not\in \mc{I}_n.
\end{cases}
\]

We have,
\begin{equation}
\begin{aligned}
\sum_{i \in \mc{I}_n}\|w_{i_{\hat{\alpha}^{GL}_n}}\| & \leq \frac{ R_n(\alpha^*_n) - R_n(\hat{\alpha}^{GL}_n)}{\zeta_n} + \sum_{i \notin \mc{I}_n}  \left( \|w_{i_{\alpha^*_n}}\|  - \|w_{i_{\hat{\alpha}^{GL}_n}}\|  \right) \\
&\leq 2  C_3 \frac{\log n}{\zeta_n \sqrt{n}} + \frac{R(\alpha^*_n) - R(\hat{\alpha}^{GL}_n)}{\zeta_n} +  \sum_{i \notin \mc{I}_n} \left |  \|w_{i_{\alpha^*_n}}\|  - \|w_{i_{\hat{\alpha}^{GL}_n}}\| \right | \\
& \leq C_3 \frac{\log n}{\zeta_n \sqrt{n}} +  \sum_{i \notin \mc{I}_n} \left |  \|w_{i_{\alpha^*_n}}\|  - \|w_{i_{\hat{\alpha}^{GL}_n}}\| \right |
\end{aligned}
\label{eq:2}
\end{equation}
with probability at least $1 - \delta$.

On the other hand, $d(\hat{\alpha}^{GL}_n, \pi_n(\hat{\alpha}^{GL}_n)) \geq \| w_{i_{\hat{\alpha}^{GL}_n}}\|$ for all $i \in \mc{I}_n$.
So,
\begin{equation}
\lambda_n \sum_{i \in \mc{I}_n} \frac{\| w_{i_{\hat \alpha_n}} \|}{d(\hat{\alpha}^{GL}_n, \pi_n(\hat{\alpha}^{GL}_n)) ^\gamma} \leq \lambda_n \sum_{i \in \mc{I}_n} \frac{\|w_{i_{\hat \alpha_n}}\|}{\| w_{i_{\hat{\alpha}^{GL}_n}}\|^\gamma} \leq R_n(\pi_n(\hat \alpha_n))  -  R_n(\hat \alpha_n).
\label{eq:4}
\end{equation}
By Lemma \ref{lem:Lip}, we obtain
\begin{equation}
R_n(\pi_n(\hat \alpha_n))  -  R_n(\hat \alpha_n) \leq C_6 d(\hat \alpha_n, \pi_n(\hat \alpha_n)) \leq C_6 \sum_{i \in \mc{I}_n}{\| w_{i_{\hat \alpha_n}} \|}
\label{eq:Lip}
\end{equation}
with probability at least $1 - \delta$.

Assume that $\sum_{i \in \mc{I}_n}\|w_{i_{\hat{\alpha}_n}}\| > 0$.
From \eqref{eq:2} and Theorem \ref{thm:gl}, we have
\begin{equation}
\begin{aligned}
d(\hat{\alpha}^{GL}_n, \pi_n(\hat{\alpha}^{GL}_n)) &\leq \sum_{i \in \mc{I}_n}\|w_{i_{\hat{\alpha}^{GL}_n}}\| \leq 2  C_3 \frac{\log n}{\zeta_n \sqrt{n}} +  \sum_{i \notin \mc{I}_n} \left |  \|w_{i_{\alpha^*_n}}\|  - \|w_{i_{\hat{\alpha}^{GL}_n}}\| \right | \\
& \leq 2  C_3 \frac{\log n}{\zeta_n \sqrt{n}} + \sqrt{H} d(\hat{\alpha}^{GL}_n, \alpha^*_n) = C_3 \frac{\log n}{\zeta_n \sqrt{n}} + \sqrt{H} d(\hat{\alpha}^{GL}_n, \mc{\overline H}^*) \\
& \leq 2  C_3(1 + \sqrt{H}) \frac{\log n}{\zeta_n \sqrt{n}} + \sqrt{H}(1 + \sqrt{H}) \left ( \frac{4C_3}{C_4} \frac{\log n}{\sqrt{n}} + \frac{2C_7}{C_4} \zeta_n^{\nu/(\nu - 1)} \right )^{1/\nu}
\end{aligned}
\label{eq:5}
\end{equation}
with probability at least $1 - \delta$.

Combining \eqref{eq:4}, \eqref{eq:Lip}, and \eqref{eq:5}, we get
\[
\frac{\lambda_n^{1/\gamma}}{2  C_3 \frac{\log n}{\zeta_n \sqrt{n}} + \sqrt{H} \left ( \frac{4C_3}{C_4} \frac{\log n}{\sqrt{n}} + \frac{2C_7}{C_4} \zeta_n^{\nu/(\nu - 1)} \right )^{1/\nu}} \leq (1 + \sqrt{H}) C_6^{1/\gamma}
\]
which is a contradiction when $n$ is large enough.
Therefore, $\sum_{i \in \mc{I}_n}\|w_{i_{\hat{\alpha}_n}}\| = 0$ with probability at least $1 - \delta$. 
Hence, every node in $\mc{I}_n$ is a zero node of $\hat{\alpha}_n$.
Recall that $\mc{I}_n$ is the set of zero nodes of $f_{\alpha^*_n}$ where $\alpha^*_n \in \mc{\overline H}^*$.
So, $| \mc{I}_n | = H - H^*$.
That is, $f_{\hat \alpha_n}$ has at most $H^*$ non-zero nodes
\end{proof}

Finally, we derive the convergence rate of the Adaptive group Lasso.

\begin{Theorem}
Assume that $\zeta_n \to 0$, $\zeta_n \sqrt{n}/ \log(n) \to \infty$, $\lambda_n \to 0$, and
\[
\frac{\lambda_n^{1/\gamma}}{2  C_3 \frac{\log n}{\zeta_n \sqrt{n}} + \sqrt{H} \left ( \frac{4C_3}{C_4} \frac{\log n}{\sqrt{n}} + \frac{2C_7}{C_4} \zeta_n^{\nu/(\nu - 1)} \right )^{1/\nu}} \to \infty.
\]
For any $\delta >0$, when $n$ is sufficiently large, we have
\[
d(\hat \alpha_n, \mc{\overline H}^*) \leq \left ( \frac {2 C_3}{C_4} \frac{\log n}{\sqrt{n}} + \frac{C_8}{C_4} \lambda_n \right )^{1/\nu}.
\]
with probability at least $1 - \delta$.
\label{thm:agl}
\end{Theorem}

\begin{proof}
Let $\alpha^*_n \in \mc{\overline H}^*$ such that $d(\hat{\alpha}^{GL}_n, \alpha^*_n) = d(\hat{\alpha}^{GL}_n, \mc{\overline H}^*)$.
By Lemma \ref{lem:generalization}, we have, for $\delta > 0$,
\begin{align*}
R(\hat \alpha_n) - R(\alpha^*_n) &\le R_n(\hat \alpha_n) - R_n(\alpha^*_n) + 2 C_3 \frac{\log n}{\sqrt{n}} \\
&\le 2 C_3 \frac{\log n}{\sqrt{n}} + \lambda_n  [M_n(\alpha^*_n) -M_n(\hat \alpha_n)] \\
&\leq 2 C_3 \frac{\log n}{\sqrt{n}} + \lambda_n M_n(\alpha^*_n).
\end{align*}
with probability at least $1 - \delta$.

Since $\mc{\overline H}^*$ is finite, by Theorem \ref{thm:gl}, there exists $C_8$ such that when $n$ is sufficiently large, $M_n(\alpha^*_n ) \leq C_8$ with probability at least $1- \delta$.
By Lemma \ref{lem:ineq}, when $n$ is sufficiently large
\[
C_4 d(\hat \alpha_n, \mc{Q})^\nu \leq R(\hat \alpha_n) - R(\alpha^*_n) \leq 2  C_3 \frac{\log n}{\sqrt{n}} + C_8 \lambda_n
\]
with probability at least $1 - \delta$.

Let $\xi_n \in \mc{Q}$ such that $d(\hat \alpha_n, \xi_n) = d(\hat \alpha_n, \mc{Q})$.
We now prove that $\xi_n \in \mc{\overline H}^*$ when $n$ is sufficiently large.
Let $\mc{I}_n$ be the set of zero node of $f_{\alpha^*_n}$.
In the proof of Theorem \ref{thm:nodes}, we prove that nodes in $\mc{I}_n$ are zero nodes of $f_{\hat{\alpha}_n}$ when $n$ is sufficiently large.
Therefore,
\begin{equation}
\| w_{i_{\xi_n}} \| \leq d(\hat \alpha_n, \xi_n) = d(\hat \alpha_n, \mc{Q}) \leq \left ( \frac {2 C_3}{C_4} \frac{\log n}{\sqrt{n}} + \frac{C_8}{C_4} \lambda_n \right )^{1/\nu}
\label{eq:6}
\end{equation}
for all $i \in \mc{I}_n$.

Consider $\xi_n' \in \mc{H}$ such that
\[
w_{i_{\xi'_n}} = 
\begin{cases}
w_{i_{\xi_n}} & \text{if } i \notin \mc{I}_n \\
0 & \text{if } i \in \mc{I}_n.
\end{cases}
\]

By Lemma \ref{lem:collapsable} and \eqref{eq:6}, when $n$ is large enough, we have $\xi_n' \in \mc{\overline H}^*$ with probability at least $1 - \delta$.
Note that $w_{i_{\hat \alpha_n}} = 0$ for all $i \in \mc{I}_n$. 
So, 
\[
d(\hat \alpha_n, \mc{Q}) \leq d(\hat \alpha_n, \xi_n') \leq d(\hat \alpha_n, \xi_n) = d(\hat \alpha_n, \mc{Q}).
\]
Thus, $w_{i_{\xi_n}} = 0$ for all $i \in \mc{I}_n$.
That is, we have $\xi_n \in \mc{\overline H}^*$.
Therefore,
\[
d(\hat \alpha_n, \mc{\overline H}^*) = d(\hat \alpha_n, \mc{Q}) \leq \left ( \frac {2 C_3}{C_4} \frac{\log n}{\sqrt{n}} + \frac{C_8}{C_4} \lambda_n \right )^{1/\nu}.
\]
\end{proof}

Combining Theorems \ref{thm:nodes} and \ref{thm:agl} we can easily show that Adaptive group Lasso estimator is model selection consistent with appropriate choice of regularizer parameters $\zeta_n$ and $\lambda_n$.
Specifically,

\begin{Theorem}
If $\zeta_n \to 0$, $\zeta_n \sqrt{n}/ \log(n) \to \infty$, $\lambda_n \to 0$, and
\[
\frac{\lambda_n^{1/\gamma}}{2  C_3 \frac{\log n}{\zeta_n \sqrt{n}} + \sqrt{H} \left ( \frac{4C_3}{C_4} \frac{\log n}{\sqrt{n}} + \frac{2C_7}{C_4} \zeta_n^{\nu/(\nu - 1)} \right )^{1/\nu}} \to \infty,
\]
then the Adaptive group Lasso estimator is model selection consistent.
\label{thm:final}
\end{Theorem}

\begin{proof}
Consistency comes directly from Theorem \ref{thm:agl}.
By Theorem \ref{thm:nodes}, the Adaptive group Lasso has at most $H^*$ non-zero nodes.
When $n$ is sufficiently large, Lemma \ref{lem:minimal} guarantees that $f_{\hat \alpha_n}$ is a minimal network.
So, it has exactly $H^*$ non-zero nodes.
\end{proof}

\section{Illustration}
\label{sec:sim}

The goal of this section is not to provide an extensive study on the performance of the destructive technique.
This has been done thoroughly in \cite{lecun1989optimal,murray2015auto, alvarez2016learning, scardapane2017group, huang2018condensenet} using both synthetic and real data.
Instead, we aim to illustrate our theoretical findings through simple experiments.

\subsection{Simulation}

\begin{figure}[ht]
\vskip 0.2in
\begin{center}
\centerline{\includegraphics[width=0.5\columnwidth]{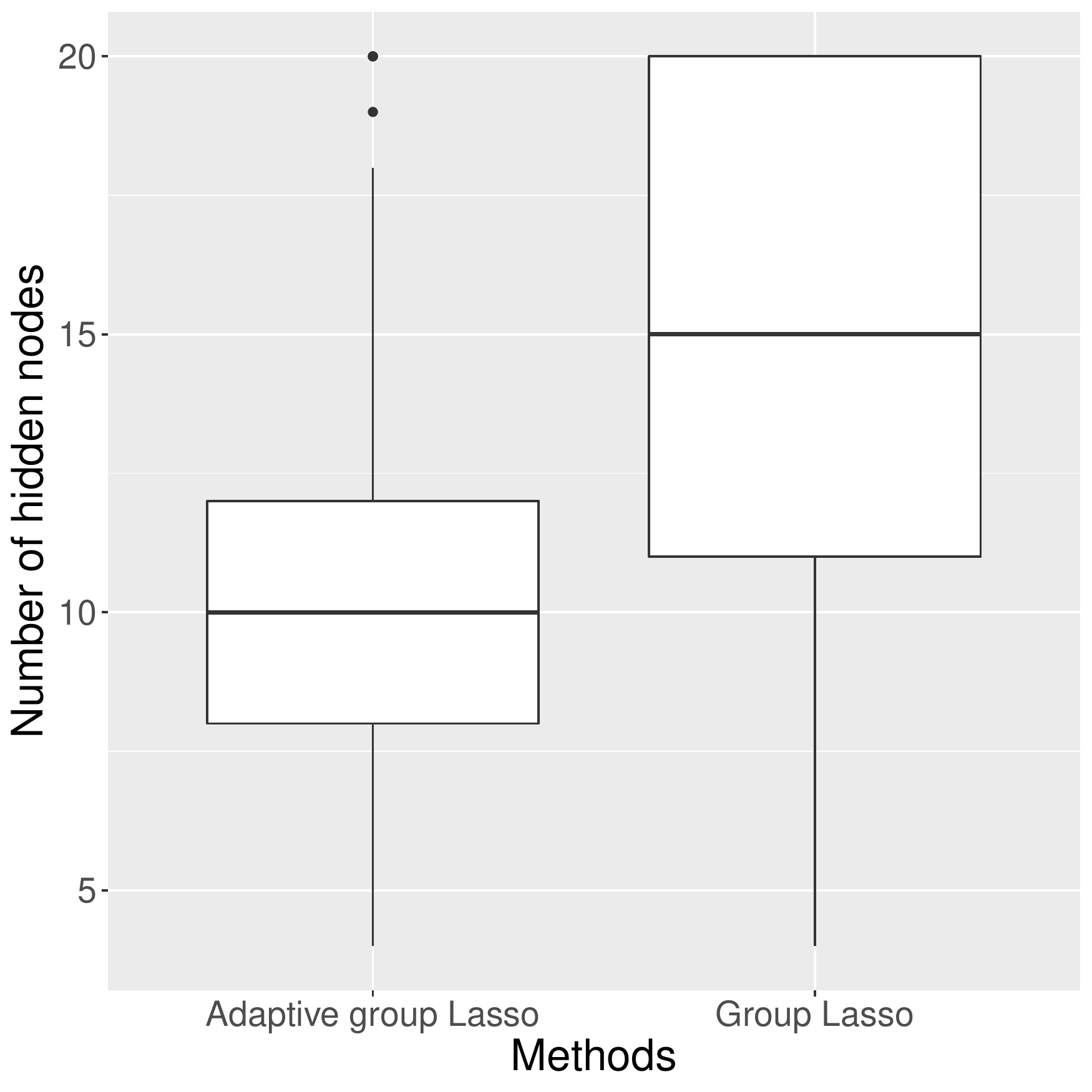}}
\caption{Number of hidden nodes selected by our proposed Adaptive Group Lasso method and the Group Lasso method proposed in \cite{murray2015auto}. The number of hidden nodes of the data-generating network is $10$.}
\label{fig:fig}
\end{center}
\vskip -0.2in
\end{figure}

In this experiment, we simulate $100$ datasets of size $n = 5000$ according to the model \eqref{eqn:model} with $H^* = 10$, $\sigma^2 = 1$.
The inputs ($5$ features) and parameters of the model are drawn independently from the standard normal distribution.
For each dataset, we train a network with $H = 20$ over $10000$ epochs using our proposed Adaptive Group Lasso method and the Group Lasso method proposed in \cite{murray2015auto, murray2019autosizing}.
The regularizing constants of both methods is chosen from the set $\{ 0.001, 0.005, 0.01, 0.025, 0.05, 0.075, 0.1 \}$ using the Akaike information criterion (AIC).
Our optimization method is Proximal gradient method \cite{parikh2014proximal} (the learning rate is $0.01$), which can identify the support of the estimates directly without the need of thresholding.
For the Adaptive group Lasso, we choose $\gamma = 2$.
The simulation is implemented in Python using Pytorch library.

We count the number of hidden nodes selected by each method.
Figure \ref{fig:fig} summarizes the results of our simulation.
The Adaptive Group Lasso performs better than the group Lasso in choosing the size of a network.
Note that the best number of hidden nodes is $10$.

\subsection{Boston housing dataset}

\begin{figure}[ht]
\vskip 0.2in
\begin{center}
\centerline{\includegraphics[width=0.5\columnwidth]{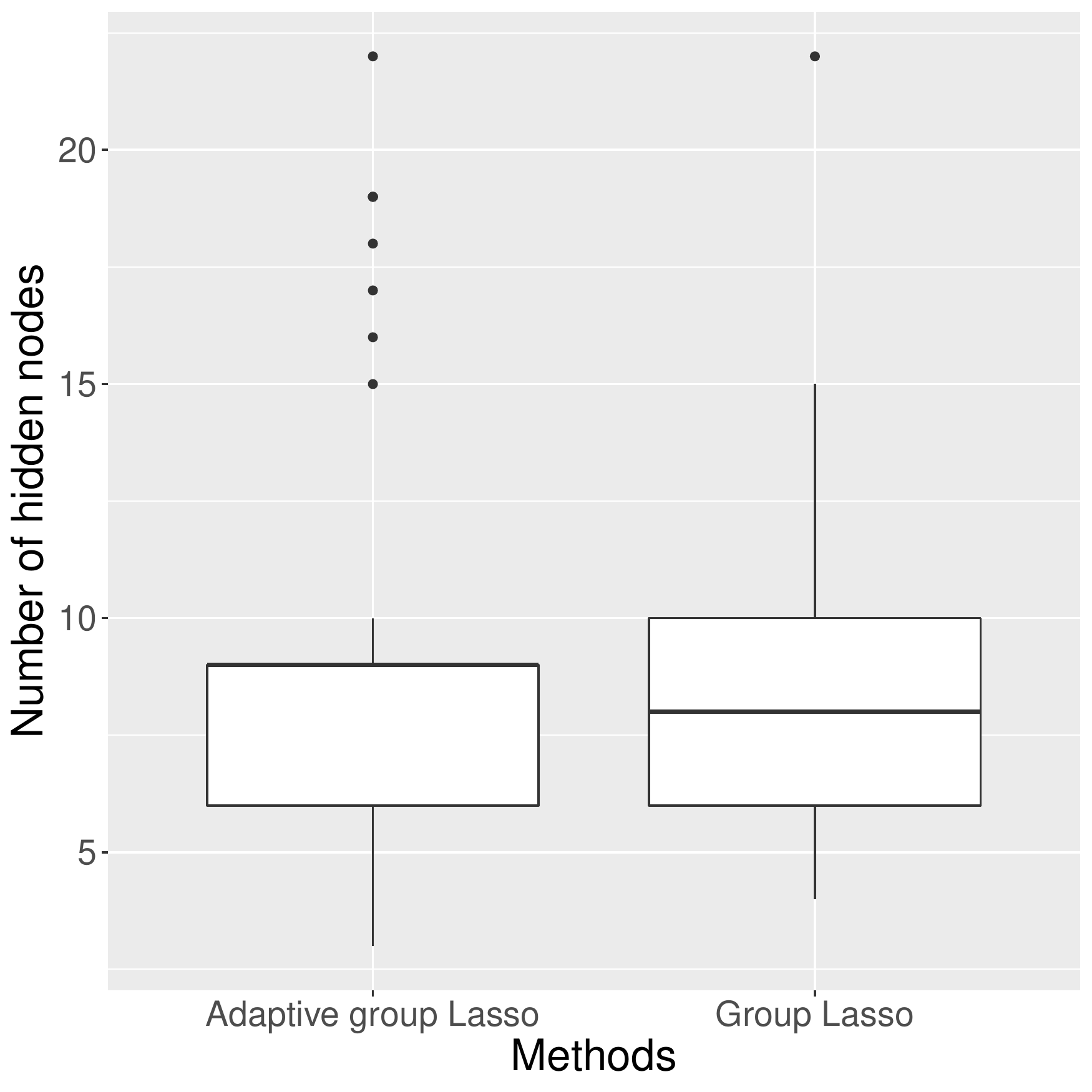}}
\caption{Number of hidden nodes selected by our proposed Adaptive Group Lasso method and the group Lasso method proposed in \cite{murray2015auto} (Boston housing dataset).}
\label{fig:fig2}
\end{center}
\vskip -0.2in
\end{figure}

\begin{figure}[ht]
\vskip 0.2in
\begin{center}
\centerline{\includegraphics[width=0.5\columnwidth]{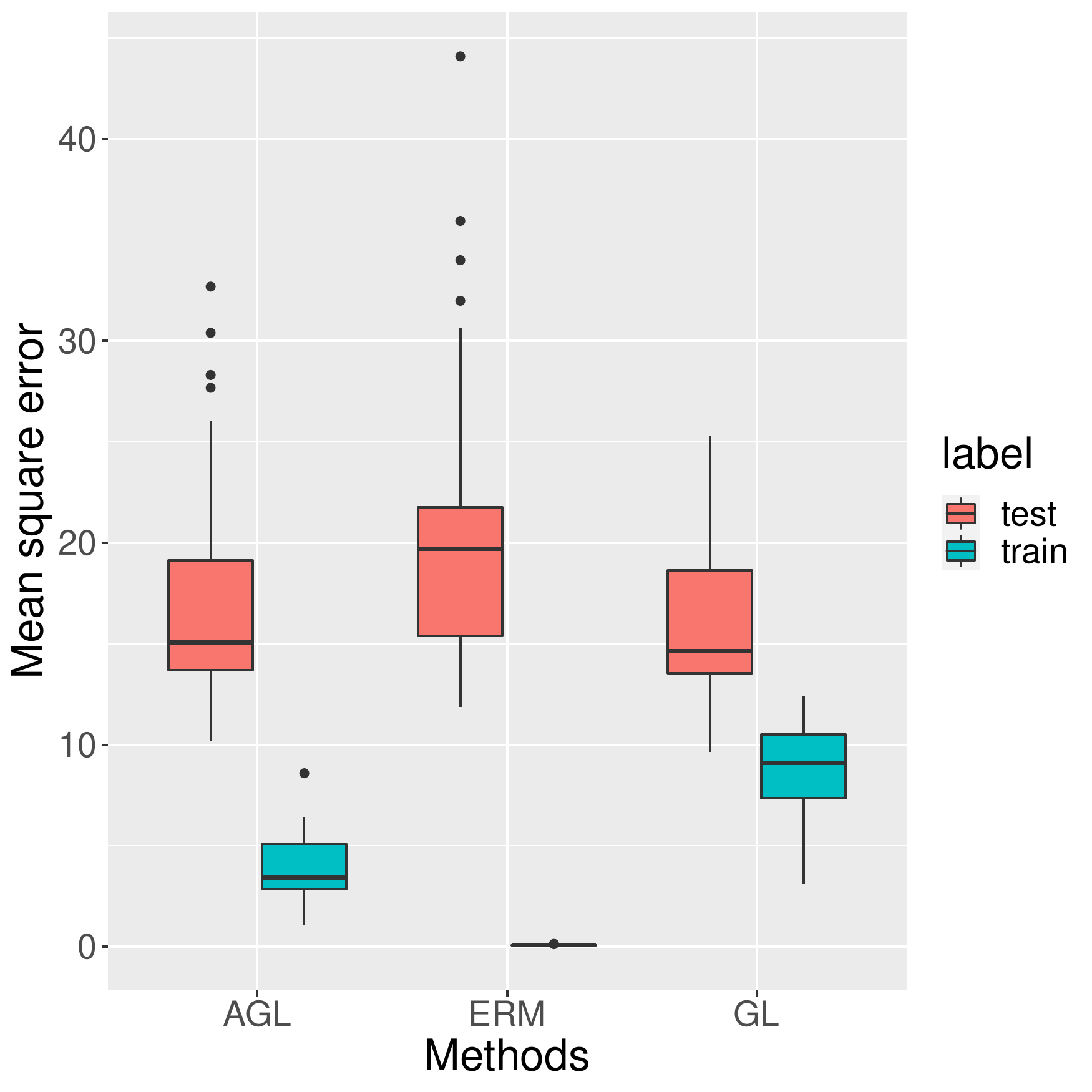}}
\caption{The training and testing errors of the Adaptive Group Lasso methods and the group Lasso method proposed in \cite{murray2015auto} (Boston housing dataset).}
\label{fig:fig3}
\end{center}
\vskip -0.2in
\end{figure}

We use our framework to study the Boston housing dataset \footnote{http://lib.stat.cmu.edu/datasets/boston}. 
This dataset consists of 506 observations of house prices and $13$ predictors. 
We consider a network with $50$ hidden nodes. 
The group Lasso and Adaptive group Lasso methods are then performed on this dataset using average test errors from 50 random train-test splits (with the size of the test sets being $25\%$ of the original dataset) over $10000$ epochs. 
The regularizing constants of the algorithms are chosen from the set $\{ 0.1, 0.3, 0.5, 0.7,1 \}$ using the Akaike information criterion (AIC).
As in the previous part, we use the proximal gradient method (with learning rate $0.01$) for optimization and choose $\gamma = 2$ for the Adaptive group Lasso. 
We also consider the simple Empirical risk minimizer (ERM) in this experiment.

The number of hidden nodes selected by group Lasso and Adaptive group Lasso are presented in Figure \ref{fig:fig2}.
Although the destructive methods choose much smaller networks (about one-fifth the size of the full networks), their prediction errors are slightly better than the ERM which uses the full network (Figure \ref{fig:fig3}).
The gap between training error and testing error of the destructive methods is also smaller compared to the ERM.

\section{Discussion and Conclusion}

We prove that our proposed Adaptive group Lasso method is model selection consistent for the problem of selecting the number of hidden nodes of one-hidden-layer feedforward networks.
To the best of our knowledge, this is the first theoretical result for the popular destructive technique.
We also obtain the consistency of the group Lasso method as a byproduct of our proof.
However, the question about the model selection consistency of the group Lasso estimator remains open.
One interesting direction for future work is extending our results to deep neural networks.
This requires further investigation on the properties of minimal deep neural networks.
Another avenue for future direction is developing theory and methods for applying constructive and destructive approaches to select the number of layers of deep neural networks.

In this paper, we assume that the true underlying function is a neural network model (Equation \ref{eqn:model}).
The extension from model-based framework to the general cases with model mismatch is an intriguing question in learning with neural networks.
In general, the projections of the true underlying function to the hypothesis space (in $\ell_2$ distance for regression) might not be unique, and they might not be similar to each other in terms of the structure of interest. Understanding of these projections for neural networks is limited, and analyses of the general cases need to involve imposing certain strong conditions on them \citep{feng2017sparse}. 
For our problem of structure reconstruction, one possible set of conditions are: (1) the set of optimal projections (in function space) is finite, and (2) all optimal projections have the same number of hidden nodes.
Our proofs can be adapted to this setting with minor adjustments.

\section*{Acknowledgement}
LSTH was supported by startup funds from Dalhousie University, the Canada Research Chairs program, the
NSERC Discovery Grant RGPIN-2018-05447, and the NSERC Discovery Launch Supplement DGECR-2018-00181.
VD was supported by a startup fund from University of Delaware and National Science Foundation grant DMS-1951474.

\bibliographystyle{chicago}
\bibliography{mybibfile}

\end{document}